\documentclass{article}

\usepackage[final]{exait_2025}

\usepackage[utf8]{inputenc} 
\usepackage[T1]{fontenc}    
\usepackage{hyperref}       
\usepackage{url}            
\usepackage{booktabs}       
\usepackage{amsfonts}       
\usepackage{nicefrac}       
\usepackage{microtype}      
\usepackage{xcolor}         
\usepackage{amsmath}
\usepackage{amsthm}
\usepackage{amssymb}
\usepackage{graphicx}
\usepackage{color}

\newcommand{\E}{\mathbb{E}}
\newtheorem{theorem}{Theorem}

\title{Prompts Generalize with Low Data: Non-vacuous Generalization Bounds for Optimizing Prompts with More Informative Priors}

\author{%
  David Madras \\
  Google Deepmind \\
  \texttt{dmadras@google.com} \\
  \And
  Joshua Safyan \\
  Google \\
  \texttt{safyan@google.com} \\
  \And
  Qiuyi (Richard) Zhang \\
  Google Deepmind \\
  \texttt{qiuyiz@google.com} \\
}

\begin{document}

\maketitle

\begin{abstract}
Many prompt engineering techniques have been successful in practice, even when optimizing over a large prompt space with with a small amount of task-specific data.
Recent work \citep{akinwande2023understanding} has partially explained this success by showing generalization bounds which apply PAC-Bayes theory to the discrete prompt space, but they are non-vacuous only in data-rich scenarios. 
We argue that such widespread success can be more fully explained through more carefully considering data- or distribution-dependent perplexity, which acts as an effective prior and steers the optimization towards prompts that are more ``natural'' for the task at hand. We derive novel generalization bounds that are non-vacuous for data-scarce prompt optimization via more useful priors, formally analyzing how perplexity regularization tightens these bounds by limiting exploration. Empirically, we explore both the bounds' effectiveness and the practical benefits of perplexity regularization in improving prompt generalization.
\end{abstract}

\section{Introduction}

Large Language Models (LLMs) gave rise to a paradigm shift in artificial intelligence, demonstrating remarkable capabilities across a wide spectrum of natural language understanding and generation tasks \citep{brown2020language}. Rather than relying solely on extensive fine-tuning for specific applications, interaction with these powerful models is increasingly mediated through \textit{prompts} – carefully crafted inputs designed to elicit desired outputs \citep{liu2023pre, shin2020autoprompt}. This has given rise to the field of \textit{prompt engineering}, optimizing prompts to effectively guide LLMs towards generating specific, high-quality responses \citep{pryzant2023automatic,  zhou2022large}.

These techniques span a range of approaches, from manual or greedy iterative refinement \citep{he2024does, zou2023universal} and automated LLM-based search over discrete prompt spaces using evolutionary algorithms \citep{pryzant2023automatic, yang2023large, shin2020autoprompt} to gradient-based optimization of continuous "soft" prompts in the embedding space \citep{li2021prefix, lester2021power}.

While these optimization methods can yield prompts that perform exceptionally well on the train data, a fundamental question remains: how do these optimized prompts \textit{generalize}, especially since the space of all prompts is large? This question of generalization is paramount for the reliable deployment of LLMs guided by optimized prompts as we move towards notions of AGI \cite{morris2023levels}. The rapid proliferation of diverse prompt optimization strategies for deployment in multi-agent, multi-modal systems has significantly outpaced our theoretical understanding of the conditions under which these methods yield generalizable solutions \cite{chen2022prompt, li2021prefix, zhou2025multi, gao2020making}. This gap motivates the need for rigorous theoretical analysis, particularly through the lens of generalization bounds, specifically tailored to the unique setting of prompt optimization.
Our main contributions are:

\begin{enumerate}
    \item \textbf{Data-Dependent Generalization Bounds:} We derive novel generalization bounds for prompt optimization algorithms operating under data scarcity. These bounds are designed to be non-vacuous, providing meaningful theoretical guarantees even for low-data settings and rely on a PAC-Bayes mechanism based on data-dependent prompt perplexity.
    \item \textbf{Empirical Bound-Regularized Prompt Optimization: } We utilize these generalization bounds as optimization objective for a prompt optimization algorithm and choose regularization priors that are informative and/or optimized on available task data. We present empirical results that validate the non-vacuous nature of our derived bounds and demonstrate the practical effectiveness of more informative perplexity regularization in improving the generalization of optimized prompts.
    
\end{enumerate}

\section{Background and Related Work}

\subsection{Prompt Optimization and Generalization}
Large Language Models (LLMs) are increasingly guided via \textit{prompts}—input sequences engineered to elicit specific behaviors \citep{liu2023pre, brown2020language}. \textit{Prompt optimization} refers to the search for a discrete token sequence (a "hard prompt") $p^*$ from a vast space of possible prompts $\mathcal{P}$, such that $p^*$ minimizes some task-specific loss when processed by a given LLM. Optimization methods range from manual tuning \citep{reynolds2021prompt} to automated techniques like greedy or evolutionary search over $\mathcal{P}$ \citep{shin2020autoprompt, pryzant2023automatic}, and selection of optimal few-shot exemplars \citep{liu2021makes}. A central challenge is ensuring that a prompt $p_{train}$, optimized on a finite training set $S$, generalizes well to unseen data, i.e., its population risk $R(p_{train})$ is close to its empirical risk $\hat{R}_S(p_{train})$. This is particularly acute in data-poor settings where $\hat{R}_S(p_{train})$ may be an unreliable estimate and overfitting is prevalent \citep{akinwande2023understanding}.

\subsection{Existing Generalization Bounds}

Generalization bounds provide mathematically rigorous guarantees on the expected performance of a learned model on unseen data, typically by relating the observable empirical performance on the training set to the unobservable true population performance \cite{vapnik1998statistical}. However, deriving meaningful generalization bounds for modern deep neural networks, especially LLMs, has proven notoriously difficult. Classical bounds based on complexity measures like VC-dimension or Rademacher complexity, as well as many standard applications of PAC-Bayes theory to model weights, often yield \textit{vacuous} results given the dataset size \cite{zhang2016understanding, jiang2019fantastic,  dziugaite2017computing}. A vacuous bound is one that provides an upper limit on the true error that is trivial (e.g., greater than 100\% error for a classification task) and thus offers no practical insight. This vacuousness is largely attributed to the immense number of parameters in these models, i.e. overparameterization, and the nature of the loss functions employed during training, such as the unbounded negative log-likelihood (NLL) \cite{zhang2016understanding, lotfi2023non}.

Deriving meaningful, i.e., \textit{non-vacuous} (non-trivial, offering actual predictive power), generalization bounds for LLMs is challenging. When considering LLM \textit{weights}, extreme overparameterization often renders classical complexity measures vacuous \citep{zhang2016understanding, jiang2019fantastic}.
In addition to its billions of model parameters, modern LLMs also have huge vocabulary sizes ($>10^6$) and large context lengths ($>10^8$), which together means that the space of all possible prompts is also incredibly large and optimizing over this space can lead to overfitting. 
If we consider the \textit{prompt space} $\mathcal{P}$, while discrete, its sheer size ($|\mathcal{V}|^L$ for vocabulary $\mathcal{V}$ and length $L$) can make uniform convergence bounds or PAC-Bayes bounds with uninformative priors vacuous, especially for small $m$. The question is whether more structured priors can mitigate this. Unbounded loss functions (e.g., NLL) also complicate standard analyses \citep{lotfi2023non}, although for many downstream tasks (e.g., classification), 0-1 loss is bounded, simplifying this aspect for prompt evaluation.

\paragraph{PAC-Bayes Bounds}
The PAC-Bayes framework provides a robust methodology for deriving generalization bounds \citep{mcallester1999pac, Alquier_2024}. 

For a hypothesis space $\mathcal{H}$ (here, this is equal to the prompt space $\mathcal{P}$), a prior distribution $P$ over $\mathcal{H}$, and a posterior distribution $Q$ (often concentrated on the learned hypothesis), a typical PAC-Bayes bound states that for any $\delta > 0$, with probability at least $1-\delta$ over the random draw of $m$ training samples $S$:
$$ R(Q) \leq \hat{R}_S(Q) + \sqrt{\frac{KL(Q||P) + \mathcal{L}(\delta, m)}{2m}} $$
where $R(Q)$ is the expected population risk under $Q$, $\hat{R}_S(Q)$ is its empirical counterpart, $KL(Q||P) = \int Q(h) \log \frac{Q(h)}{P(h)} dh$ is the Kullback-Leibler divergence, and $\mathcal{L}(\delta, m)$ is a term like $\log \frac{m}{\delta}$ or $\log \frac{|\mathcal{H}|}{\delta}$ for finite spaces. The $KL$ term penalizes posteriors distant from the prior, acting as a regularizer; a low $KL$ implies the posterior is "compressible" relative to the prior. While $P$ is traditionally data-independent, PAC-Bayes also accommodates \textit{data-dependent priors} \citep{parrado2012pac}, where $P$ is chosen based on some data $S_{prior}$ (e.g., a subset of $S$ or auxiliary unlabeled data), provided $S_{prior}$ is handled appropriately to avoid invalidating the bound (e.g., $S_{prior}$ is distinct from the data $S_{val}$ used for $\hat{R}_S(Q)$, or its influence is accounted for in $\mathcal{L}$) \citep{negrea2019information, catoni2007}. Such priors can yield tighter bounds if they effectively capture data-specific structures.


\subsection{Prior Work on Prompt Performance Guarantees}

A significant advancement in understanding generalization for prompt optimization methods was made by \cite{akinwande2023understanding}. They demonstrated that by applying PAC-Bayes bounds not to the LLM's weights, but rather to the \textit{discrete hypothesis space of prompts}, and by utilizing another LLM to define a prior distribution over these prompts, it is possible to obtain non-vacuous generalization bounds. Their approach yielded bounds that were remarkably tight (often within a few percentage points of the true test error) for tasks performed in \textit{data-rich} settings, such as zero-shot classification on ImageNet using CLIP prompts. However, this leaves open the critical question of why prompts generalize even when optimized in \textit{data-poor} scenarios, a common situation in practice where users might only have a handful of examples to tune a prompt for a specific task \cite{chen2022prompt, wang2023promptagent}.
\citet{akinwande2023understanding} applied PAC-Bayes to discrete prompts, using a data-independent prior $P(\text{prompt})$ derived from another LLM's likelihood for that prompt sequence. For a chosen prompt $p_{opt}$ (a Dirac delta posterior $Q$), $KL(Q||P)$ becomes $-\log P(p_{opt})$. They achieved non-vacuous bounds in data-rich settings, where $\hat{R}_S(p_{opt})$ is reliable. However, the reliance on a strong empirical risk estimate and a fixed, data-independent LLM prior limits applicability in data-scarce regimes where $P(p_{opt})$ might be low for an overfit, atypical prompt.

Empirical work by \citet{gonen2022demystifying} demonstrated a negative correlation between a prompt's task-contextualized perplexity (average negative log-likelihood given unlabeled task inputs) and its downstream performance. Lower perplexity prompts (more "natural" or probable to the LLM in context) tended to perform better. While this suggests perplexity as a valuable heuristic, it doesn't provide formal generalization guarantees.

Prompt Risk Control (PRC) \citep{zollo2024prompt} offers an alternative data-dependent guarantee. Using a validation set $S_{val}$, PRC applies Distribution-Free Uncertainty Quantification (DFUQ) to provide high-probability upper bounds on pre-specified risk measures (e.g., mean loss, Value-at-Risk). For instance, PRC might guarantee that $P(R_{VaR@0.95}(\text{metric}) \le \alpha) \ge 1-\delta$. Its data-dependency stems from $S_{val}$ used to calibrate the risk bound $\alpha$. This contrasts with PAC-Bayes, which typically bounds the expected population loss $R(Q)$ or its deviation from $\hat{R}_S(Q)$. While both PRC and our proposed approach leverage validation data, PRC focuses on controlling specific statistics of the loss distribution observed on $S_{val}$, whereas our aim is to bound the generalization error by incorporating data-derived information into the PAC-Bayes prior itself.

\subsection{Motivation for Data-Dependent, Perplexity-Informed PAC-Bayes Bounds}


The challenges inherent in data-poor regimes necessitate analysis that exploit prior information that lies within the pre-trained LLM itself, reflecting the vast amounts of knowledge implicitly encoded during its initial training phase. One way to access this information is through \textit{perplexity}, a standard metric in language modeling that measures how well a model predicts a given text sequence \citep{meister2021language, shannon1948mathematical}. Gonen et al. provided empirical evidence that prompts which the model found less perplexing tended to yield generally better results, suggesting that perplexity captures an intrinsic quality of the prompt related to the model's ability to process and execute the instruction effectively, independent of task-specific labels \cite{gonen2022demystifying}. Recently, similar ideas involving perplexity-based regularization have shown promise in related areas, such as prompt learning for vision-language models \citep{liu2024plpp}. Perplexity is also a widely-used metric that has a wide range of applications, such as detecting adversarial attacks \citep{alon2023detecting}, pruning \cite{ankner2024perplexed}, and uncertainty estimation \citep{cooper2024perplexed}.

The limitations of existing approaches in data-poor settings motivate our work. If, as \citet{gonen2022demystifying} suggest, low perplexity is indicative of good prompts, then a PAC-Bayes prior $P$ that assigns higher probability to low-perplexity prompts could yield tighter bounds. If this prior $P$ is itself shaped or selected based on data (e.g., unlabeled task data, or a held-out portion of labeled data to estimate perplexities), it becomes a data-dependent prior. The core hypothesis is that such a perplexity-informed, data-dependent prior can more effectively constrain the $KL(Q||P)$ term for prompts that generalize well, even when $m$ is small, leading to non-vacuous bounds where data-independent or uninformative priors might fail. This paper aims to formalize this intuition and derive such bounds.

\section{Data-Dependent Prompt Generalization Bounds}

To address the limitations of existing generalization bounds in the context of data-poor prompt optimization, we can draw inspiration from previous work on data-dependent generalization bounds. Again, the standard PAC-Bayes bound is as follows: Let $\mathcal{H}$ be a hypothesis space, $P$ a prior distribution over $\mathcal{H}$, and $S$ a sample of size $m$ drawn from a data distribution $D$. For any $\delta > 0$, with probability at least $1-\delta$ over the random choice of $S$, the following holds for  in expectation for $h\sim Q$ and draws of $S$:

$$R(Q) \leq \hat{R}_S(Q) + \sqrt{\frac{KL(Q||P) + \log \frac{m}{\delta}}{2m}} $$

where for some loss function $l$, $R(h) = \underset{W\sim D}{\E}[l(h, W)]$ is the population risk , $\hat{R}_S(h) = \frac{1}{m}\underset{w\in S}{\sum} l(h, w)$ is the empirical risk, and $KL(Q||P)$ is the Kullback-Leibler divergence between the posterior distribution $Q$ and the prior distribution $P$. Note that this expected bound can be derandomized at a cost of a mild increase in generalization bound to hold for all $h \in \mathcal{H}$.

There are many forms of PAC-Bayes bounds available \citep{alquier2024user}.
Another noteworthy one is the Tolstikhin and Seldin bound \citep{tolstikhin2013pac}, which we use in our experimental section later on.
This bound has the desirable property of depending only on $1 / m$ if the estimator has training error equal to 0.
This bound is as follows:

$$R(Q) \leq \hat{R}_S(Q) + \sqrt{2 \hat{R}_S(Q) \frac{KL(Q||P) + \log{\frac{2 \sqrt{m}}{\delta}}}{m}} + 2 \frac{KL(Q||P) + \log{\frac{2 \sqrt{m}}{\delta}}}{m}$$


In the context of prompt optimization, we can consider the hypothesis space $\mathcal{H}$ to be the set of all possible prompts and even when we define our prior to be given by an LLM's loglikelihood, this still results in vacuous bounds for mild regimes of $m < 1000$ as the divergence is large when applying the previous approaches of \cite{akinwande2023understanding}. The key idea for tighter bounds is to use a non-empty prompt prior but rather to allow for an optimized data-dependent prior, exploiting the compressive power of an LLM to significantly reduce the size of our hypothesis search space.

We adapt the approach of \cite{negrea2019information} to our hypothesis space $\mathcal{H}$ as the set of all discrete prompts. Let $P(h | p) = P_{LLM}(h|p)$ be a prior distribution over prompts $h \in \mathcal{H}$ given by a conditional distribution, conditioned on some prior prompt $p$, produced by an LLM. A data-dependent prior can be constructed by optimizing $p(J)$ given by a sample of the data $J \subset S \sim\mathcal{D}^n$, such that $p(J)$ ideally minimizes $KL(Q || P(h|p(J)))$, where $Q$ is the posterior distribution over task-optimized prompts. Using the data-dependent prior defined above, we can derive a PAC-Bayes generalization bound for prompt optimization. We defer the proof details to the appendix.

\begin{theorem}[Data-Dependent PAC-Bayes Prompt Bound]
\label{thm:main}
Let $S \sim \mathcal{D}^n$ be an i.i.d. data sample of size $n$ and $J \subset S$ be a uniform subset of $S$ of size $m < n$. Suppose the underlying loss $l(h, W)$ is $\sigma$-subgaussian for $W \sim \mathcal{D}$, and any prompt $h$. Then, for any $\sigma(J)$-measurable prior $P$ and $\sigma(S)$-measurable posterior distribution $Q$ if $h\sim Q$, then 

$$\E_S\left[\left|R(h) - \hat{R}_S(h)\right|\right] \leq \sqrt{2\frac{\sigma^2}{n-m} \E_S[KL(Q||P)]} $$

Specifically, given some prior prompt $p(J) \in \mathcal{H}$ that is LLM-optimized on $J$ and $Q = \{(q_1, q_2, ..., q_k)\}$ a discrete set of $\sigma(S)$-measurable optimized task prompts, then with high constant probability,

$$\left|R(q_j) - \hat{R}_S(q_j)\right| \leq O\left(\sqrt{\frac{\sigma^2}{n-m} \left[ -\log(k) - \frac{1}{k} \sum_{i=1}^k \log P_{LLM}(q_i | p)\right]}\right) $$

for at least some $q_j \in Q$ and $P_{LLM}(q|p)$ measures perplexity of prompt $q$ conditioned on $p$ using any LLM.
\end{theorem}

In the theorem stated, the KL divergence term now measures how close the posterior distribution over prompts is to the data-dependent prior, which still has a relatively wide spread across the prompt space. Therefore, it is not surprising that by using an stochastic posterior (i.e. a uniform posterior over $k$ prompts), we can get a $-\log(k)$ dependency in the generalization bound, with posteriors that give similar log probabilities. In practice, we can improve the PAC-Bayes bounds significantly by choosing $k$ precisely, although in our later sections, we set $k = 1$ and even with this naive setting of $k$, surprisingly we are able to derive non-vacuous bounds by only exploiting data dependence.
\section{Results}

In this section, we assess the practical utility of using a more useful prior prompt for a PAC-Bayes generalization bound and demonstrate that it gives non-vacuous bounds for a simple classification task.
We explore two potential improvements to specify the prior distribution:
\begin{enumerate}
    \item Using an \textit{informative} prompt, created manually
    \item Using a \textit{data-dependent} prompt, learned from previous successful prompts
\end{enumerate}

Using a real text classification dataset, we attempt to optimize the previously proposed PAC-Bayes generalization bounds (from \citep{tolstikhin2013pac}) and observe how tight it can become. The dataset we use is the ETHOS Hate Speech dataset, from \citep{mollas2022ethos}.
This dataset, which is publicly available, uses comments  collected from various online sources, and annotates them with a binary ``Yes/No'' label as to whether or not they qualify as hate speech. We use the Gemini models, specifically 2.0 Flash, for our LLM experiments  \citep{team2023gemini}.
This is the type of task one might apply hand-tuned prompt engineering to in order to develop a performative prompt.
However, in order to more systematically search the prompt space, we use the automated prompt optimization (APO) algorithm from \citep{pryzant2023automatic}, which optimizes a prompt for a given dataset through iteratively applying edits based on past rounds of prediction.

Importantly, we specify that the APO should optimize on the generalization bound described previously - at every step, we calculate the bound outlined in Theorem~\ref{thm:main} with $k = 1$, and aim to minimize the loss upper bound. Specifically, for this task, we maximize the empirical accuracy loss with the generalization upper bound. We run this procedure for 200 steps, using a 90\% error bound.
Due to the bandit-style structure of APO, the more accurate prompts are usually tested on more examples - this can result in better bounds for the more accurate prompts purely due to larger $n$ being tested by the APO algorithm. 
In order to level the playing field, we also show below ``n-adjusted'' versions of the error bounds, where we re-calculate the bound with the maximum value of $n$ across the 4 prompts ($n=160$). This is not necessarily the ``correct'' value since, had each prompt actually been tested on 160 examples, the training error may have changed - but it does provide a more level playing field to compare the bounds in an optimization-algorithm-independent manner.

In order to calculate the bound with an LLM, we need a prior - this is defined by some text that the LLM can condition on when determining what the log-likelihood of a prompt is.
We can think of this prior prompt as a ``meta-prompt'': helping us to define the distribution of task prompts we will consider, exploiting powerful priors that are encoded within the LLM. 
We test three prompts: the first is the empty string and the second is \texttt{``We are trying to find classification labels for hate speech detection. <empty line> The text of the prompt is as follows: <empty line>''.} 
This second prompt is the ``informative'' prior in Table \ref{table:results}, below.
The third prompt is listed as ``optimized'' in the Prior column of the table below --- this prompt is data-dependent, and is found by optimizing the log-likelihood of several prompts discovered on a previous, accuracy-based run of APO. This prompt is longer and can be found in the Appendix.


The prompts themselves are:
\begin{itemize}
    \item Handcrafted: \texttt{``Does this input contain hate speech?''}
    \item Optimized (for bound, using empty prior): \texttt{`` Is this message hateful or discriminatory?''}
    \item Optimized (for bound, using meta-prompt prior): \texttt{``Is this hate speech?''}
    \item Optimized (for accuracy, using meta-prompt prior): \texttt{``Does this statement express hatred or prejudice towards a specific group based on characteristics such as race, religion, ethnicity, sexual orientation, etc., with the intent to cause harm or marginalize?''}
    \item Optimized (for bound, using optimized data-dependent prior): \texttt{``Does this statement contain hate speech? (Yes/No)''}

\end{itemize}

\begin{table}[]
\centering
\begin{tabular}{|c|c|c|c|c|c|c|}
\hline
\textbf{Prompting Method} & \textbf{Prior} & \textbf{Train Error} & \textbf{Log-lik.} & \textbf{Test Error} & \textbf{Bound} & \textbf{Bound (n-adj)} \\ \hline
handcrafted     & empty          & 0.2                  & -39.569                 & 0.145               & 1.977          & 0.818                       \\ \hline
handcrafted     & informative    & 0.175                & -26.936                 & 0.145               & 1.497          & 0.644                       \\ \hline
optimized (acc) & informative   & \textbf{0.087}                & -83.258                 & 0.141               & 1.63           & 0.953                       \\ \hline
optimized       & empty          & 0.133                & -44.504                 & 0.149               & 0.882          & 0.734                       \\ \hline
optimized       & informative   & 0.131                & \textbf{-17.414 }                & 0.112               & \textbf{0.468} & \textbf{0.468}              \\ \hline
optimized       & optimized    & 0.133               & -28.885               &       \textbf{0.104}     &   0.695 &    0.587         \\ \hline
\end{tabular}
\caption{Results from automated prior optimization to minimize 90\% error bounds on Hate Speech dataset.
See text for values of prompts and priors.
``optimized (acc)'' means that prompt was optimized for accuracy; ``optimized'' means that the prompt was optimized for the generalization bound.
$n$-adjusted bound is for easier comparison only --- the ``true'' bound outputted by the optimization is in the ``Bound'' column.}
\label{table:results}
\end{table}

In Table \ref{table:results}, we observe that the optimized prompts with non-empty priors result in the tightest bounds, around 0.46.
While this is not necessarily low enough to  be useful (real test error was 0.11), we note that frequently generalization bounds in the domain of deep learning are totally trivial (i.e. > 1), and that a bound < 0.5 represents a promising step in the right direction.
Additionally, we note that the tighter bounds in \citet{akinwande2023understanding} are obtained using a much larger set of data (CIFAR-10, $\approx$ 10k examples), whereas we obtained our non-trivial results using only 150-300 examples.
This demonstrates the value of using handcrafted priors in the prompt optimization process.

Additionally, we notice that the optimization of the generalization bound with both non-empty meta-prompts improves the test error over all other methods; in particular, it achieves better test error than just optimizing for accuracy (the ``optimized (acc)'' row).
The data-dependent prior achieves the best test error, and the hand-engineered meta-prompt achieves the second best test error in the table above.
We note that all test errors have overlapping 95\% confidence intervals, so we should be careful about drawing too-strong conclusions.
However, we believe this is a promising result, suggesting that the perplexity regularization inherent in the form of the generalization bound may yield practical improvements in robustness and overfitting, even if the bound itself is not low enough to be useful independently.

Finally, we highlight how the results in Table \ref{table:results} should be interpreted in light of \citet{akinwande2023understanding}.
The rows corresponding most closely to their proposed methods are (handcrafted, empty) and (optimized, empty).
We note several key distinctions as to why the bounds in our table are empirically looser than theirs:
\begin{itemize}
    \item We use an order of magnitude less data --- O(100) vs O(1000)
    \item We use a less precise optimization algorithm (APO), which treats the LLM as a black box at the prompt level, whereas \citet{akinwande2023understanding} uses greedy token-by-token sampling on the bound itself
\end{itemize}

With this context, we note there is a favorable comparison to the ``empty prior'' approach from \citet{akinwande2023understanding}, and that further gains may yet be realized by using an informative or data-dependent prior alongside more advanced optimization methods.

\section{Conclusion}
This research demonstrates the value of informative perplexity regularization in achieving non-vacuous generalization bounds for prompt optimization, particularly in data-scarce environments. By leveraging priors conditioned on task information or data, the proposed PAC-Bayes bounds offer meaningful guarantees where traditional methods often fail for LLMs. The empirical results, though modest, show that these bounds can be tightened significantly compared to uninformative priors and that optimizing for these bounds can lead to improved test error. We believe that experimentation on other datasets will lead to similar results and in addition to extending the breadth of our research, we  generally hope also extend the depth of in future work, such as:

\begin{enumerate}
    \item {\bf Complex Prior Optimization:} This includes exploring hierarchical priors, embedded priors, or prior optimization techniques that adapt more dynamically to the nuances of the task data. The current work showed that even a simple "meta-prompt" can improve bounds; more sophisticated, learnable prior-generation mechanisms could further enhance performance.
    \item {\bf Custom Algorithms for Regularized Prompt Optimization:} Developing custom optimization algorithms specifically designed for the proposed regularized objective is a key next step. Current methods, like APO, were adapted for this research. New algorithms could more directly incorporate the perplexity regularization and the structure of the generalization bound into the search process, potentially leading to more efficient and effective optimization, especially when aiming to minimize the bound itself. This would also involve a more systematic exploration of the parameter k in the generalization bound, which was naively set to 1 in the current experiments.
    \item {\bf Stochastic Posterior Prompts:} The theoretical framework already suggests benefits for using an stochastic posterior (a uniform posterior over $k$ prompts) through a $-\log(k)$ dependency in the generalization bound. Future work will investigate practical methods for creating and utilizing ensembles of prompts, guided by classical techniques such as boosting.
\end{enumerate}

\bibliographystyle{plainnat}
\bibliography{ref}
\appendix

\section{Proofs}

\begin{proof}[Proof of ~\ref{thm:main}]

Our first statement follows from applying the data-dependent mutual information bound \cite{negrea2019information} to the prompt space with any prior. Now, note that if we define $p(J)$ as a specific LLM-optimized prompt, then it is measurable with respect to $J$ as long as each iteration of the optimization procedure is measurable. This is indeed the case since applying LLM optimization is a measurable function, modeled by a transformer model.

Now, let $S$ be our data sample and if our prior is simply one prompt $p(J)$ and the posterior is simply some discrete set of $\sigma(S)$-measurable task optimized prompts given by $Q = \{(q_1, q_2, ..., q_k)\}$. Then, let $\widetilde{Q}$ be the uniform distribution over $Q$ and consider the prior $\widetilde{P}$ as the conditional distribution of any LLM considitioned on $p(J)$, the $\sigma(J)$-measurable prior prompt. We can now express our bound as that for $h \sim \widetilde{Q}$, our first statement gives us that

$$\E_S\left[\left|R(h) - \hat{R}_S(h)\right|\right] \leq \sqrt{2\frac{\sigma^2}{n-m} \E_S[KL(\widetilde{Q}||\widetilde{P})]} $$

By definition, we can rewrite the last term in the expression as $KL(\widetilde{Q} || P_{LLM}(h|p(J)))$, where $p(J)$ is the $\sigma(J)$-measurable prior prompt. Finally, by substituting in $\widetilde{Q}$,

$$KL(\widetilde{Q} || P(h|p(S))) =  -\log(k) - \frac{1}{k} \sum_{i=1}^k \log(P_{LLM}(q_i | p(S)))$$

where $P_{LLM}(x_i | p)$ is the probability assigned by the LLM to the input $x_i$ given the prompt $p$. Therefore, for some $q_i$, it must hold that in expectation

$$\left|R(q_j) - \hat{R}_S(q_j)\right| \leq O\left(\sqrt{\frac{\sigma^2}{n-m} \left[ -\log(k) - \frac{1}{k} \sum_{i=1}^k \log P_{LLM}(q_i | p)\right]}\right) $$

Finally, our second statement follows by Markov's inequality. 
\end{proof}

\section{Optimized Prior Prompt}

Here we give the data-dependent prompt used in Table \ref{table:results}, listed as the ``optimized'' prior.
This prompt was found by running APO to maximize the log-likelihood of four prompts discovered from a previous run on this task.
The prompt is:
\texttt{Create a hate speech classification rubric utilizing a decision tree approach. \
The rubric will begin with a primary question determining the presence of hate speech (yes/no). \
If yes, subsequent questions will assess severity (low, medium, high) \
based on linguistic features (slurs, dehumanizing language), \
contextual factors (platform, audience, intent), and target specificity \
(clearly identified group).  Each question will have clearly defined criteria \
and branching pathways leading to a final severity classification. \
The rubric will include multiple examples illustrating the decision-making process, \
differentiating hate speech from strong criticism and addressing potential biases.}

\end{document}